\documentclass[english,10pt]{amsart}
\usepackage[margin=1in]{geometry}

\usepackage{amsmath}
\usepackage{amssymb}
\usepackage{amsthm}
\usepackage{babel}
\usepackage{graphicx}
\usepackage{color}
\usepackage{url}
\usepackage{tikz}
\usetikzlibrary{arrows}

\newtheorem{theorem}{Theorem}[section]

\newtheorem{lemma}[theorem]{Lemma}
\newtheorem{proposition}[theorem]{Proposition}
\newtheorem{corollary}[theorem]{Corollary}

\newtheorem*{rare-eclipse-problem}{Rare Eclipse Problem}
\newtheorem*{gordons-comparison-theorem}{Gordon's Comparison Theorem}
\newtheorem*{gordons-escape-theorem}{Gordon's Escape Through a Mesh Theorem}

\begin{document}

\thispagestyle{empty}

\title{Compressive classification and the rare eclipse problem}

\author[Bandeira]{Afonso S.~Bandeira}
\address[Bandeira]{Program in Applied and Computational Mathematics (PACM), Princeton University, Princeton, NJ 08544, USA ({\tt ajsb@math.princeton.edu}).}

\author[Mixon]{Dustin G. Mixon}
\address[Mixon]{Department of Mathematics and Statistics, Air Force Institute of Technology, Wright-Patterson AFB OH, USA ({\tt dustin.mixon@afit.edu}).}

\author[Recht]{Benjamin Recht}
\address[Recht]{Department of Electrical Engineering and Computer Science, Department of Statistics, University of California, Berkeley CA, USA ({\tt brecht@eecs.berkeley.edu}).}

\begin{abstract}
This paper addresses the fundamental question of when convex sets remain disjoint after random projection.
We provide an analysis using ideas from high-dimensional convex geometry.
For ellipsoids, we provide a bound in terms of the distance between these ellipsoids and simple functions of their polynomial coefficients.
As an application, this theorem provides bounds for compressive classification of convex sets.
Rather than assuming that the data to be classified is sparse, our results show that the data can be acquired via very few measurements yet will remain linearly separable.
We demonstrate the feasibility of this approach in the context of hyperspectral imaging.
\end{abstract}

\maketitle

\section{Introduction}

A decade of powerful results in compressed sensing and related fields have demonstrated that many signals that have low-dimensional latent structure can be recovered from very few compressive measurements.
Building on this work, many researchers have shown that classification tasks can also be run on compressive measurements, provided that either the data or classifier is sparse in an appropriate basis~\cite{HauptEtal:06,DavenportEtal:07,DuarteEtal:07,BoufounosB:08,GuptaNR:10,PlanV:13}.
However, classification is a considerably simpler task than reconstruction, as there may be a large number of hyperplanes which successfully cleave the same data set.
The question remains:

\begin{quote}
Can we successfully classify data from even fewer compressive measurements than required for signal reconstruction?
\end{quote}

Prior work on compressive classification has focused on preserving distances or inner products between data points.
Indeed, since popular classifiers including the support vector machine and logistic regression only depend on dot products between data points, it makes sense that if dot products are preserved under a compressive measurement, then the resulting decision hyperplane should be close to the one computed on the uncompressed data.

In this paper, we take a different view of the compressive classification problem, and for some special cases, we are able to show that data can be classified with extremely few compressive measurements.
Specifically, we assume that our data classes are circumscribed by disjoint convex bodies, and we seek to avoid intersection between distinct classes after projection.
By studying the set of separating hyperplanes, we provide a general way to estimate the minimal dimension under which two bodies remain disjoint after random projection.
In Section~3, we specialize these results to study ellipsoidal classes and give our main theoretical result---that $k$ ellipsoids of sufficient pairwise separation remain separated after randomly projecting onto $O(\log k)$ dimensions.
Here, the geometry of the ellipsoids plays an interesting and intuitive role in the notion of sufficient separation.
Our results differ from prior work insofar as they can be applied to \textit{full} dimensional data sets and are independent of the number of points in each class.
We provide a comparison with principal component analysis in Section~4 by considering different toy examples of classes to illustrate strengths and weaknesses, and then by applying both approaches to hyperspectral imaging data.  We conclude in Section~5 with a discussion of future work.

\section{Our Model and Related Work}

In this section, we discuss our model for the classes as well as the underlying assumptions we apply throughout this paper.
Consider an ensemble of classes $C_i\subseteq\mathbb{R}^N$ that we would like to classify.
We assume that these classes are pairwise \textit{linearly separable}, that is, for every pair $i,j$ with $i\neq j$, there exists a hyperplane in $\mathbb{R}^N$ which separates $C_i$ and $C_j$.
Equivalently, we assume that the convex hulls $S_i:=\operatorname{hull}(C_i)$ are disjoint, and for simplicity, we assume these convex hulls are closed sets.

Linear separability is a particularly useful property in the context of classification, since to demonstrate non-membership, it suffices to threshold an inner product with the vector normal to a separating hyperplane.
Of course, in many applications, classes do not enjoy this (strong) property, but the property can be weakened to \textit{near} linear separability, in which there exists a hyperplane that mostly distinguishes of a pair of classes.
One may also lift to a tensored version of the vector space and find linear separability there.
Since linear separability is so useful, we use this property as the basis for our notion of distortion:
We seek to project the classes $\{C_i\}_{i=1}^k$ in such a way that their images are linearly separable.

Our assumptions on the $C_i$'s and our notion of distortion both lead to a rather natural problem in convex geometry (see Figure~\ref{figure.rare eclipse problem} for an illustration):

\begin{rare-eclipse-problem}
Given a pair of disjoint closed convex sets $A,B\subseteq\mathbb{R}^N$ and $\eta>0$, find the smallest $M$ such that a random $M\times N$ projection $P$ satisfies $PA\cap PB=\emptyset$ with probability $\geq1-\eta$.
\end{rare-eclipse-problem}

\begin{figure}
\begin{center}
\includegraphics[width=0.5\textwidth]{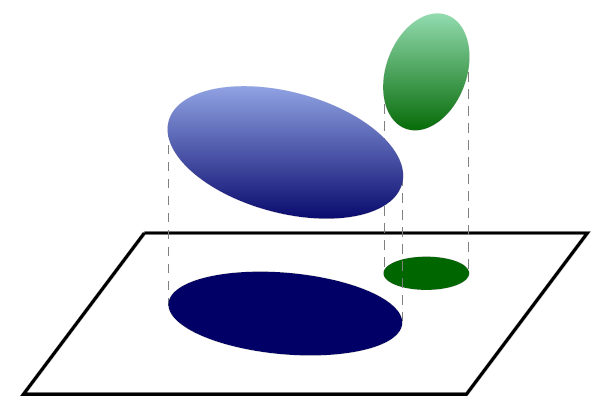}
\end{center}
%\begin{center}
%\begin{tikzpicture}[scale=0.5]
%\draw [very thick] (-5,-3) -- (6,-3) -- (3,-7) -- (-8,-7) -- (-5,-3);
%\draw [dashed,gray] (0.95,0.6) -- (0.95,-4);
%\draw [dashed,gray] (3.05,1.4) -- (3.05,-4);
%\shade [top color=blue!30!green!40!white, bottom color=green!40!black] [shift={(2,1)},rotate=70] ellipse [x radius=1.5cm, y radius=1cm];
%\filldraw [green!40!black] [shift={(2,-4)},rotate=0] ellipse [x radius=1.05cm, y radius=0.4cm];
%\draw [dashed,gray] (-4.4,-0.4) -- (-4.4,-4.8);
%\draw [dashed,gray] (1.4,-1.6) -- (1.4,-5.2);
%\shade [top color=green!20!blue!40!white, bottom color=blue!40!black] [shift={(-1.5,-1)},rotate=-15] (0,0) ellipse [x radius=3cm, y radius=1.5cm];
%\filldraw [blue!40!black] [shift={(-1.5,-5)},rotate=-5] (0,0) ellipse [x radius=2.9cm, y radius=1cm];
%\end{tikzpicture}
%\end{center}
\caption{
\label{figure.rare eclipse problem}
{\small 
Two sufficiently separated convex sets remain separated when projected onto a subspace.
The Rare Eclipse Problem asks for the smallest $M$ such that this happens when projecting onto a random subspace of dimension $M$.
Solving this problem for a given ensemble of classes enables dimensionality reduction in a way that ensures linear separability for classification.
}\normalsize}
\end{figure}

At this point, we discuss some related work in the community.
It appears that compressive classification was studied as early as 2006, when~\cite{HauptEtal:06} considered a model in which each class is a point in Euclidean space.
Interestingly, this bears some resemblance to the celebrated work in~\cite{IndykM:98,Kleinberg:97}, which used random projections to quickly approximate nearest-neighbor search.
The work in~\cite{DavenportEtal:07,DuarteEtal:07} considered a more exotic family of classes, namely low-dimensional manifolds---this is particularly applicable to the classification of images according to the primary object featured in each image.
Along these lines of low-dimensional classes, there has since been some work in the case where classes are low-dimensional subspaces~\cite{MajumdarW:10,ReboredoRCR:14}, or unions thereof~\cite{AryafarJS:12}.
Specifically, \cite{ReboredoRCR:14} considers a Gaussian mixture model in which each Gaussian is supported on a different subspace.  From a slightly dual view, researchers have also shown that if the classifier is known to be sparse, then we can subsample the data itself, and the separating hyperplane can be determined from a number of examples roughly proportional to the sparsity of the hyperplane~\cite{BoufounosB:08,GuptaNR:10,PlanV:13}.

It is striking that, to date, all of the work in compressive classification has focused on classes of low dimension.
This is perhaps an artifact of the mindset of compressed sensing, in which the projection preserves all information on coordinate planes of sufficiently small dimension.
However, classification should not require nearly as much information as signal reconstruction does, and so we expect to be able to compressively classify into classes of full dimension; indeed, we allow two points in a common class to be mapped to the same compressive measurement, as this will not affect the classification.
A Boolean version of this idea is studied in \cite{AbbeBCC:14}, which considers both random and optimality constructed projections.
In the continuous setting, the closest existing work is that of Dasgupta~\cite{Dasgupta:99,Dasgupta:00}, which uses random projections to learn a mixture of Gaussians.
In particular, Dasgupta shows that sufficiently separated Gaussians stay separated after random projection.
In the next section, we prove a similar result about ellipsoids, but with a sharper notion of separation.

\section{Theoretical Results}

Given two disjoint closed convex bodies $A,B\subseteq\mathbb{R}^N$ and a projection dimension $M$, the Rare Eclipse Problem asks whether a random $M\times N$ projection $P$ of these bodies avoids collision, i.e., whether $PA\cap PB$ is typically empty.
This can be recast as a condition on the $(N-M)$-dimensional null space of $P$:
\begin{equation*}
PA\cap PB=\emptyset
\qquad
\Longleftrightarrow
\qquad
\operatorname{Null}(P)\cap(A-B)=\emptyset,
\end{equation*}
where $A-B$ denotes the Minkowski difference of $A$ and $B$.
Of course, the null space of $P$ is closed under scalar multiplication, and so avoiding $A-B$ is equivalent to avoiding the normalized versions of the members of $A-B$.
Indeed, if we take $S$ to denote the intersection between the unit sphere in $\mathbb{R}^N$ and the cone generated by $A-B$, then
\begin{equation*}
PA\cap PB=\emptyset
\qquad
\Longleftrightarrow
\qquad
\operatorname{Null}(P)\cap S=\emptyset.
\end{equation*}
Now suppose $P$ is drawn so that its entries are iid $\mathcal{N}(0,1)$.
Then by rotational invariance, the distribution of its null space is uniform over the Grassmannian.
As such, the Rare Eclipse Problem reduces to a classical problem in convex geometry:
Given a ``mesh'' (a closed subset of the unit sphere), how small must $K$ be for a random $K$-dimensional subspace to ``escape through the mesh,'' i.e., to avoid collision?
It turns out that for this problem, the natural way to quantify the size of a mesh is according to its \textit{Gaussian width}:
\begin{equation*}
w(S)
:=\mathbb{E}_g\bigg[\sup_{z\in S}\langle z,g\rangle\bigg],
\end{equation*}
where $g$ is a random vector with iid $\mathcal{N}(0,1)$ entries.
Indeed, Gaussian width plays a crucial role in the following result, which is an improvement to the original (Corollary~3.4 in~\cite{Gordon:88}); the proof is given in the appendix, and follows the proof of Corollary~3.3 in~\cite{ChandrasekaranRP:12} almost identically.

\begin{gordons-escape-theorem}
Take a closed subset $S$ of the unit sphere in $\mathbb{R}^N$, and denote $\lambda_M:=\mathbb{E}\|g\|_2$, where $g$ is a random $M$-dimensional vector with iid $\mathcal{N}(0,1)$ entries.
If $w(S)<\lambda_M$, then an $(N-M)$-dimensional subspace $Y$ drawn uniformly from the Grassmannian satisfies
\begin{equation*}
\operatorname{Pr}\Big(Y\cap S=\emptyset\Big)
\geq 1-\exp\bigg(-\frac{1}{2}\Big(\lambda_M-w(S)\Big)^2\bigg).
\end{equation*}
\end{gordons-escape-theorem}

It is straightforward to verify that $\lambda_M\geq\sqrt{M-1}$, and so rearranging leads to the following corollary:

\begin{corollary}
\label{corollary.Gordon}
Take disjoint closed convex sets $A,B\subseteq\mathbb{R}^N$, and let $w_\cap$ denote the Gaussian width of the intersection between the unit sphere in $\mathbb{R}^N$ and the cone generated by the Minkowski difference $A-B$.
Draw an $M\times N$ matrix $P$ with iid $\mathcal{N}(0,1)$ entries.
Then
\begin{equation*}
\begin{array}{lcl}
M>\Big(w_\cap+\sqrt{2\log(1/\eta)}\Big)^2+1 &\Longrightarrow& \operatorname{Pr}\big(PA\cap PB=\emptyset\big)\geq 1-\eta.
\end{array}
\end{equation*}
\end{corollary}

Now that we have a sufficient condition on $M$, it is natural to wonder how tight this condition is.
Recent work by Amelunxen, Lotz, McCoy and Tropp~\cite{AmelunxenLMT:13} shows that the Gordon's results are incredibly tight.
Indeed, by an immediate application Theorem~I and Proposition~10.1 in~\cite{AmelunxenLMT:13}, we achieve the following characterization of a phase transition for the Rare Eclipse Problem:

\begin{corollary}
\label{corollary.AKF}
Take disjoint closed convex sets $A,B\subseteq\mathbb{R}^N$, and let $w_\cap$ denote the Gaussian width of the intersection between the unit sphere in $\mathbb{R}^N$ and the cone generated by the Minkowski difference $A-B$.
Draw an $M\times N$ matrix $P$ with iid $\mathcal{N}(0,1)$ entries.
Then
\begin{equation*}
\begin{array}{lcl}
M\geq w_\cap^2+\sqrt{16N\log(4/\eta)}+1 &\Longrightarrow& \operatorname{Pr}\big(PA\cap PB=\emptyset\big)\geq 1-\eta,\\
M\leq w_\cap^2-\sqrt{16N\log(4/\eta)}   &\Longrightarrow& \operatorname{Pr}\big(PA\cap PB=\emptyset\big)\leq \eta,
\end{array}
\end{equation*}
\end{corollary}
Considering the second part of Corollary~\ref{corollary.AKF}, the bound in Corollary~\ref{corollary.Gordon} is essentially tight.
Also, since Corollary~\ref{corollary.AKF} features an additional $\sqrt{N}$ factor in the error term of the phase transition, the bound in Corollary~\ref{corollary.Gordon} is stronger than the first part of Corollary~\ref{corollary.AKF} when $w_\cap\ll\sqrt{N}-\sqrt{\log(1/\eta)}$, which corresponds to the regime where we can compress the most: $M\ll N$.

\subsection{The case of two balls}

Corollaries~\ref{corollary.Gordon} and~\ref{corollary.AKF} demonstrate the significance of Gaussian width to the Rare Eclipse Problem.
In this subsection, we observe these quantities to solve the Rare Eclipse Problem in the special case where $A$ and $B$ are balls.
Since each ball has its own parameters (namely, its center and radius), in this subsection, it is more convenient to write $A=S_1$ and $B=S_2$.
The following lemma completely characterizes the difference cone $S_1-S_2$:

\begin{lemma}
\label{lemma.circcone lower bound}
For $i=1,2$, take balls $S_i:=\{c_i+r_ix:x\in \mathcal{B}\}$, where $c_i\in\mathbb{R}^N$, $r_i>0$ such that $r_1+r_2<\|c_1-c_2\|$, and $\mathcal{B}$ denotes the ball centered at $0$ of radius $1$.
Then the cone generated by the Minkowski difference $S_1-S_2$ is the circular cone
\begin{equation*}
\operatorname{Circ}(\alpha):=\{z:\langle z,c_1-c_2\rangle\geq\|z\|\|c_1-c_2\|\cos\alpha\},
\end{equation*}
where $\alpha\in(0,\pi/2)$ is the angle such that $\sin\alpha=(r_1+r_2)/\|c_1-c_2\|$.
\end{lemma}

In three dimensions, the fact that the difference cone is circular makes intuitive sense.
The proof of Lemma~\ref{lemma.circcone lower bound} is routine and can be found in the appendix.

Considering the beginning on this section, it now suffices to bound the Gaussian width of the circular cone's intersection with the unit sphere $\mathbb{S}^{N-1}$.
Luckily, this computation is already available as Proposition~4.3 in~\cite{AmelunxenLMT:13}:
\begin{equation*}
\Big(w(\operatorname{Circ}(\alpha))\cap\mathbb{S}^{N-1}\Big)^2=N\sin^2\alpha+O(1).
\end{equation*}
See Figure~\ref{figure_circcone_phasetransition} for an illustration of the corresponding phase transition.
By Lemma~\ref{lemma.circcone lower bound} (and Corollaries~\ref{corollary.Gordon} and~\ref{corollary.AKF}), this means a random $M\times N$ projection will keep two balls from colliding provided
\begin{equation*}
M\geq N\bigg(\frac{r_1+r_2}{\|c_1-c_2\|}\bigg)^2+O(\sqrt{N}).
\end{equation*}
Note that there is a big payoff in the separation $\|c_1-c_2\|$ between the balls.
Indeed, doubling the separation decreases the required projection dimension by a factor of $4$.

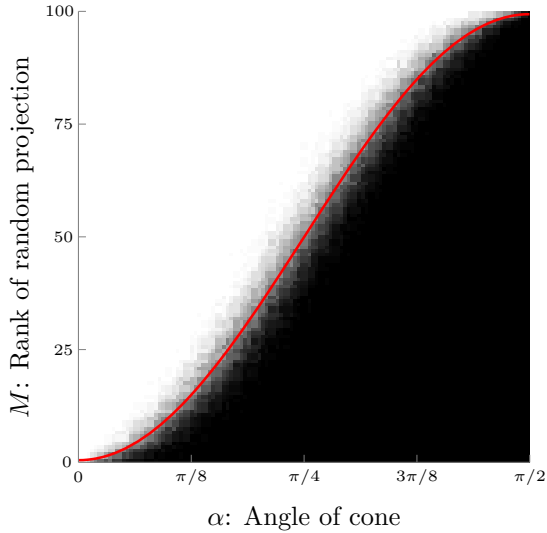
\begin{figure}[t]
\begin{center}
\begin{tikzpicture}
\node[inner sep=0pt] (plot) at (3,3) {\includegraphics[width=6cm]{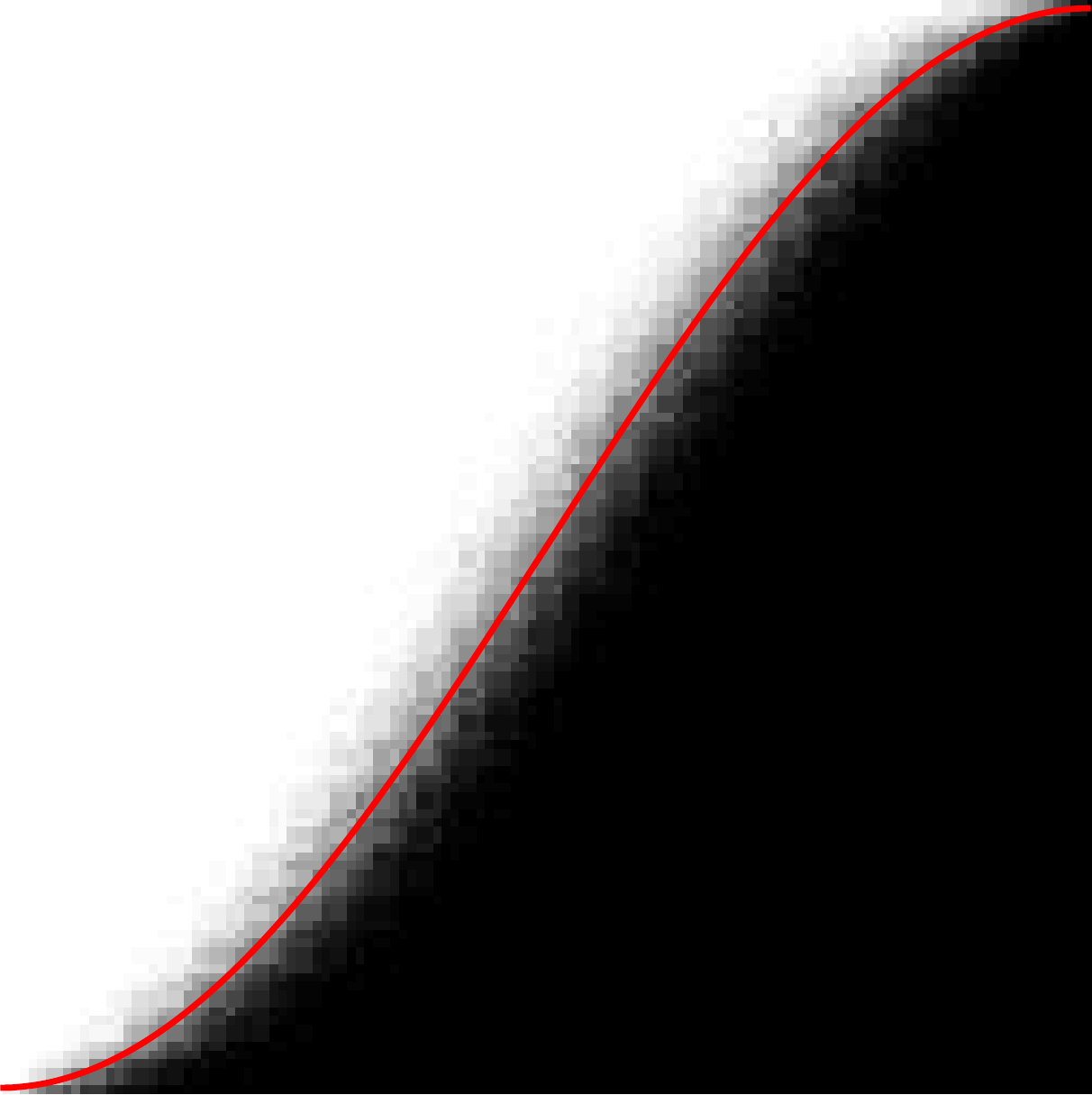}};
\draw[gray] (0,0) -- (6,0);
\draw[gray] (0,0) -- (0,6);
\draw[gray,very thin] (1.5,0) -- (1.5,0.1);
\draw[gray,very thin] (3,0) -- (3,0.1);
\draw[gray,very thin] (4.5,0) -- (4.5,0.1);
\draw[gray,very thin] (6,0) -- (6,0.1);
\draw[gray,very thin] (0,1.5) -- (0.1,1.5);
\draw[gray,very thin] (0,3) -- (0.1,3);
\draw[gray,very thin] (0,4.5) -- (0.1,4.5);
\draw[gray,very thin] (0,6) -- (0.1,6);
\draw (0,-0.2) node {\tiny{$0$}};
\draw (1.5,-0.2) node {\tiny{$\pi/8$}};
\draw (3,-0.2) node {\tiny{$\pi/4$}};
\draw (4.5,-0.2) node {\tiny{$3\pi/8$}};
\draw (6,-0.2) node {\tiny{$\pi/2$}};
\draw (3,-0.75) node {$\alpha$: Angle of cone};
\draw (-0.15,0) node {\tiny{$0$}};
\draw (-0.2,1.5) node {\tiny{$25$}};
\draw (-0.2,3) node {\tiny{$50$}};
\draw (-0.2,4.5) node {\tiny{$75$}};
\draw (-0.25,6) node {\tiny{$100$}};
\draw (-0.75,3) node [rotate=90] {$M$: Rank of random projection};
\end{tikzpicture}
\caption{
\label{figure_circcone_phasetransition}
{\small 
Phase transition for a random null space to avoid a circular cone.
Fixing the ambient dimension to be $N=100$, then for each $\alpha=1:\pi/200:\pi/2$ and $M=1:100$, we randomly drew $100$ $M\times N$ matrices with iid $\mathcal{N}(0,1)$ entries and plotted the proportion whose null spaces avoided the circular cone with angle $\alpha$.
As expected, if $\alpha$ is large, then so must $M$ so that the null space is small enough to avoid the cone.
In red, we plot the curve $M=N\sin^2\alpha+\cos2\alpha$, which captures the phase transition by Theorem~I and Proposition~4.3 in~\cite{AmelunxenLMT:13}.
By Lemma~\ref{lemma.circcone lower bound}, the circular cone is precisely the difference cone of two balls, and so this phase transition solves the Rare Eclipse Problem in this special case.
}\normalsize}
\end{center}
\end{figure}

\subsection{The case of two ellipsoids}

Now that we have solved the Rare Eclipse Problem for balls, we consider the slightly more general case of ellipsoids.
Actually, this case is somewhat representative of the general problem with arbitrary convex sets.
This can be seen by appealing to the following result of Paouris~\cite{Paouris:06}:

\begin{theorem}[Concentration of Volume]
\label{theorem.concentration of volume}
There is an absolute constant $c>0$ such that the following holds:
Given a convex set $K\subseteq\mathbb{R}^N$, draw a random vector $X$ uniformly from $K$.
Suppose $K$ has the property that $\mathbb{E}[X]=0$ and $\mathbb{E}[XX^\top]=I$.
Then
\begin{equation*}
\operatorname{Pr}\Big(\|X\|_2>r\Big)\leq e^{-cr}
\qquad
\forall r\geq\sqrt{N}.
\end{equation*}
\end{theorem}

In words, the above theorem says that the volume of an isotropic convex set is concentrated in a round ball.
The radius of the ball of concentration is $O(\sqrt{N})$, which corresponds to the fact that $\mathbb{E}\|X\|_2^2=\mathbb{E}\operatorname{Tr}[XX^\top]=N$.
This result can be modified to describe volume concentration of any convex set (isotropic or not).
To see this, consider any convex set $K\subseteq\mathbb{R}^N$ of full dimension (otherwise the volume is zero).
Then taking $Y$ to be a random vector drawn uniformly from $K$, we define the centroid $c:=\mathbb{E}[Y]$.
Also, since $K$ has full dimension, the inertia matrix $\mathbb{E}[(Y-c)(Y-c)^\top]$ is symmetric and positive definite, and we can take $A_0:=(\mathbb{E}[(Y-c)(Y-c)^\top])^{1/2}$.
It is straightforward to verify that $X:=A_0^{-1}(Y-c)$ is distributed uniformly over $K':=A_0^{-1}(K-c)$, and that $K'$ satisfies the hypotheses of Theorem~\ref{theorem.concentration of volume}.
We claim that $Y$ is concentrated in an ellipsoid defined by
\begin{equation*}
S_r:=\{c+rA_0x:x\in\mathcal{B}\}
\end{equation*}
for some $r\geq\sqrt{N}$, where $\mathcal{B}$ denotes the ball centered at $0$ of radius $1$.
Indeed, Theorem~\ref{theorem.concentration of volume} gives
\begin{equation*}
\operatorname{Pr}\Big(Y\not\in S_r\Big)
=\operatorname{Pr}\Big(\|A_0^{-1}(Y-c)\|_2>r\Big)
\leq e^{-cr}.
\end{equation*}
Overall, the vast majority of any convex set is contained in an ellipsoid defined by its centroid and inertia matrix, and so two convex sets are \textit{nearly} linearly separable if the corresponding ellipsoids are linearly separable.
(A similar argument relates the case of two ellipsoids to a mixture of two Gaussians.)

Note that any ellipsoid has the following convenient form:
\begin{equation*}
\{c+Ax:x\in\mathcal{B}\},
\end{equation*}
where $c\in\mathbb{R}^N$ is the center of the ellipsoid, $A$ is some $N\times N$ symmetric and positive semidefinite matrix, and $\mathcal{B}$ denotes the ball centered at the origin of radius $1$.
Intuitively, the difference cone of any two ellipsoids will not be circular in general, as it was in the case of two balls.
Indeed, the oblong shape of each ellipsoid (determined by its shape matrix $A$) precludes most of the useful symmetries in the difference cone, and as such, the analysis of the size of the cone is more difficult.  Still, we established the following upper bound on the Gaussian width in the general case, which by Corollaries~\ref{corollary.Gordon} and~\ref{corollary.AKF}, translates to a sufficient number of rows for a random projection to typically maintain separation:

\begin{theorem}
\label{theorem.two ellipsoids}
For $i=1,2$, take ellipsoids $S_i:=\{c_i+A_ix:x\in \mathcal{B}\}$, where $c_i\in\mathbb{R}^N$, $A_i$ is symmetric and positive semidefinite, and $\mathcal{B}$ denotes the ball centered at $0$ of radius $1$.
Let $w_\cap$ denote the Gaussian width of the intersection between the unit sphere in $\mathbb{R}^N$ and the cone generated by the Minkowski difference $S_1-S_2$.
Then
\begin{equation*}
w_\cap\leq\frac{\|A_1\|_F+\|A_2\|_F}{\zeta-\big(\|A_1e\|_2+\|A_2e\|_2\big)}+\frac{1}{\sqrt{2\pi}}
\end{equation*}
provided $\zeta>\|A_1e\|_2+\|A_2e\|_2$; here, $\zeta:=\|c_2-c_1\|$ and $e:=(c_1-c_2)/\|c_1-c_2\|$.
\end{theorem}

The proof is technical and can be found in the appendix, but the ideas behind the proof are interesting.
There are two main ingredients, the first of which is the following result:

\begin{proposition}[Proposition~3.6 in~\cite{ChandrasekaranRP:12}]
\label{proposition 3.6}
Let $\mathcal{C}$ be any non-empty convex cone in $\mathbb{R}^N$, and let $g$ be an $N$-dimensional vector with iid $\mathcal{N}(0,1)$ entries.
Then
\begin{equation*}
w(\mathcal{C}\cap\mathbb{S}^{N-1})
\leq\mathbb{E}_g\Big[\|g-\Pi_{\mathcal{C}^*}(g)\|_2\Big],
\end{equation*}
where $\Pi_{\mathcal{C}^*}$ denotes the Euclidean projection onto the dual cone $\mathcal{C}^*$ of $\mathcal{C}$.
\end{proposition}

Proposition~\ref{proposition 3.6} is essentially a statement about convex duality, and while it provides an upper bound on $w_\cap$, in our case, it is difficult to find a closed form expression for the right-hand side.
However, the bound is in terms of distance to the dual cone, and so any point in this cone provides an upper bound on this distance.
This leads to the second main ingredient in our analysis:
We choose a convenient mapping $\widetilde{\Pi}$ that sends any vector $g$ to a point in $\mathcal{C}^*$ (but not necessarily the closest point), while at the same time allowing the expectation of $\|g-\widetilde{\Pi}(g)\|_2$ to have a closed form.
Since $\|g-\Pi_{\mathcal{C}^*}(g)\|_2\leq\|g-\widetilde{\Pi}(g)\|_2$ for every possible instance of $g$, this produces a closed-form upper bound on the bound in Proposition~\ref{proposition 3.6}.

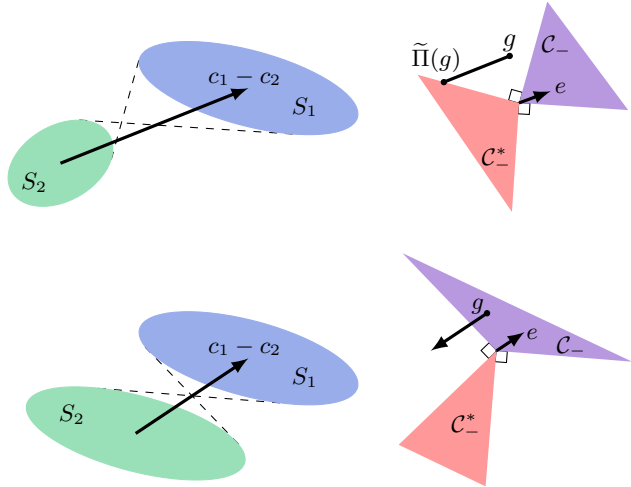
\begin{figure}
\begin{center}
\begin{tabular}{cc}
\vspace{12pt}
\hfill
\begin{tikzpicture}[scale=0.5,>=latex]
%\draw[step=.5cm] (-6,-5) grid (6,5);
\draw [dashed] (-1.15,-0.9) -- (-0.4,1.8);
\draw [dashed] (-2,0.15) -- (3.8,-0.23);
\filldraw [green!20!blue!40!white] [shift={(2.5,1)},rotate=-15] (0,0) ellipse [x radius=3cm, y radius=1cm];
\filldraw [blue!30!green!40!white] [shift={(-2.5,-1)},rotate=30] ellipse [x radius=1.5cm, y radius=1cm];
\draw [->,very thick] (-2.5,-1) -- (2.5,1);
\draw (2.4,1.3) node {\small{$c_1-c_2$}};
\draw (4,0.5) node {\small{$S_1$}};
\draw (-3.2,-1.5) node {\small{$S_2$}};
\end{tikzpicture}
&
\vspace{0pt}
\begin{tikzpicture}[scale=0.5,>=latex]
\draw [shift={(-0.01,0.01)},rotate=-93.5] (0,0) rectangle (0.3,0.3);
\draw [shift={(0.01,-0.01)},rotate=74.5] (0,0) rectangle (0.3,0.3);
\fill [red!30!blue!40!white] (0,0) -- (2.9,-0.19) -- (0.75,2.7) --cycle;
\fill [red!40!white] (0,0) -- (-0.19,-2.9) -- (-2.7,0.75) --cycle;
\draw [->,very thick] (0,0) -- (0.844,0.3125);
\filldraw [black] (-0.25,1.25) circle [radius=0.07cm];
\draw (-0.25,1.6) node {$g$};
\draw [very thick] (-0.25,1.25) -- (-2,0.553);
\filldraw [black] (-2,0.553) circle [radius=0.07cm];
\draw (-2.2,1.2) node {\small{$\widetilde{\Pi}(g)$}};
\draw (1.1,0.4) node {\small{$e$}};
\draw (1,1.5) node {\small{$\mathcal{C}_-$}};
\draw (-0.6,-1.5) node {\small{$\mathcal{C}_-^*$}};
\end{tikzpicture}
\\
\vspace{12pt}
\hfill
\begin{tikzpicture}[scale=0.5,>=latex]
%\draw[step=.5cm] (-6,-5) grid (6,5);
\draw [dashed] (-1.25,1.3) -- (1.25,-1.3);
\draw [dashed] (-2.5,0.2) -- (2.5,-0.2);
\filldraw [green!20!blue!40!white] [shift={(1.5,1)},rotate=-15] (0,0) ellipse [x radius=3cm, y radius=1cm];
\filldraw [blue!30!green!40!white] [shift={(-1.5,-1)},rotate=-15] ellipse [x radius=3cm, y radius=1cm];
\draw [->,very thick] (-1.5,-1) -- (1.5,1);
\draw (1.4,1.3) node {\small{$c_1-c_2$}};
\draw (3,0.5) node {\small{$S_1$}};
\draw (-3.2,-0.5) node {\small{$S_2$}};
\end{tikzpicture}
&
\vspace{0pt}
\begin{tikzpicture}[scale=0.5,>=latex]
\draw [shift={(-0.01,0.01)},rotate=-94.5] (0,0) rectangle (0.3,0.3);
\draw [shift={(0.015,0)},rotate=133.8] (0,0) rectangle (0.3,0.3);
\fill [red!30!blue!40!white] (0,0) -- (3.607,-0.289) -- (-2.5,2.6) --cycle;
\fill [red!40!white] (0,0) -- (-0.289,-3.607) -- (-2.6,-2.5) --cycle;
\draw [->,very thick] (0,0) -- (0.75,0.5);
\filldraw [black] (-0.25,1) circle [radius=0.07cm];
\draw (-0.5,1.2) node {\small{$g$}};
\draw [->,very thick] (-0.25,1) -- (-1.75,0);
\draw (0.95,0.45) node {\small{$e$}};
\draw (2,0.1) node {\small{$\mathcal{C}_-$}};
\draw (-0.8,-2) node {\small{$\mathcal{C}_-^*$}};
\end{tikzpicture}
\end{tabular}
\end{center}
\caption{
\label{figure.pseudoprojection}
{\small 
Two examples of two ellipsoids along with their difference cone and dual cone.
For each pair, on the left, the vector $c_1-c_2$ is depicted, as are the extreme difference directions between the ellipsoids---these form the boundary of the difference cone $\mathcal{C}_-$, which is illustrated on the right along with its dual cone $\mathcal{C}_-^*$, i.e., the cone of separating hyperplanes.
The vector $e\in\mathcal{C}_-$ is a normalized version of $c_1-c_2$.
In the first example, the pseudoprojection $\widetilde{\Pi}$ sends any point $g$ to the closest point in the dual cone $\mathcal{C}_-^*$ along the line spanned by $e$.
Interestingly, in cases where the ellipsoids are far apart, the cone $\mathcal{C}_-$ will be narrow, and so the boundary of the dual cone will essentially be the orthogonal complement of $e$.
As such, the pseudoprojection is close the true projection onto the polar cone in this limiting case.
For this pseudoprojection to be well-defined, we require that for every $g$, the line which passes through $g$ in the direction of $e$ hits the dual cone at some point.
This is not always the case, as the second example illustrates.
It is straightforward to show that this pseudoprojection is well-defined if and only if the ellipsoids remain separated when projecting onto the line spanned by $e$.
}\normalsize}
\end{figure}

Figure~\ref{figure.pseudoprojection} illustrates how we chose the pseudoprojection $\widetilde{\Pi}$.
Interestingly, this pseudoprojection behaves more like the true projection when the ellipsoids are more distant from each other.
At the other extreme, note that Theorem~\ref{theorem.two ellipsoids} does not hold if the ellipsoids are too close, i.e., if $\|c_1-c_2\|\leq\|A_1e\|_2+\|A_2e\|_2$.
This occurs, for example, if the two ellipsoids collide after projecting onto the span of $e$; indeed, taking $x$ and $y$ to be unit-norm vectors such that $e^\top(c_1+A_1x)=e^\top(c_2+A_2y)$, then rearranging gives
\begin{equation*}
\|c_1-c_2\|
=e^\top c_1-e^\top c_2
=-e^\top A_1x+e^\top A_2y
\leq|e^\top A_1x|+|e^\top A_2y|
\leq\|A_1e\|_2+\|A_2e\|_2.
\end{equation*}
As Figure~\ref{figure.pseudoprojection} illustrates, our pseudoprojection even fails to be well-defined when the ellipsoids collide after projecting onto the span of $e$.
So why bother using a random projection to maintain linear separability when there is a rank-$1$ projection available?
There are two reasons:
First, calculating this rank-$1$ projection requires access to the centers of the ellipsoids, which are not available in certain applications (e.g., unsupervised or semi-supervised learning, or if the projection occurs blindly during the data collection step).
Second, the use of a random projection is useful when projecting multiple ellipsoids simultaneously to preserve pairwise linear separability---as we will detail in the next subsection, randomness allows one to appeal to the union bound in a way that permits several ellipsoids to be projected simultaneously using particularly few projected dimensions.

At this point, we compare Theorem~\ref{theorem.two ellipsoids} to the better understood case of two balls.
In this case, $A_1=r_1I$ and $A_2=r_2I$, and so Theorem~\ref{theorem.two ellipsoids} gives that
\begin{equation*}
w_\cap\leq\sqrt{N}\cdot\frac{r_1+r_2}{\|c_1-c_2\|_2-(r_1+r_2)}+\frac{1}{\sqrt{2\pi}}.
\end{equation*}
If we consider the regime in which $r_1+r_2\leq\frac{1}{2}\|c_1-c_2\|_2$, then we recover the case of two balls to within a factor of $2$, suggesting that the analysis is tight (at least in this case).
For a slightly more general lower bound, note that a projection maintains separation between two ellipsoids only if it maintains separation between balls contained in each ellipsoid.
The radius of the largest ball in the $i$th ellipsoid is equal to the smallest eigenvalue $\lambda_\textrm{min}(A_i)$ of the shape matrix $A_i$, and the center of this ball coincides with the center $c_i$ of its parent ellipsoid.
As such, we can again appeal to the case of two balls to see that Theorem~\ref{theorem.two ellipsoids} is reasonably tight for ellipsoids of reasonably small eccentricity $\lambda_\textrm{max}(A_i)/\lambda_\textrm{min}(A_i)$.  Closed form bounds for general ellipses with high eccentricity are unwieldy, but Figure~\ref{figure_ellipsoids_phasetransition} illustrates that our bound is far from tight when the ellipsoids are close to each other.
Still, the bound improves considerably as the distance increases.  As such, we leave improvements to Theorem~\ref{theorem.two ellipsoids} as an open problem (in particular, finding a closed-form characterization of the phase transition in terms of the $c_i$'s and $A_i$'s).

\begin{figure}[t]
\begin{center}
\begin{tabular}{cc}
\begin{tikzpicture}
\node[inner sep=0pt] (plot) at (3.07,3) {\includegraphics[height=6cm]{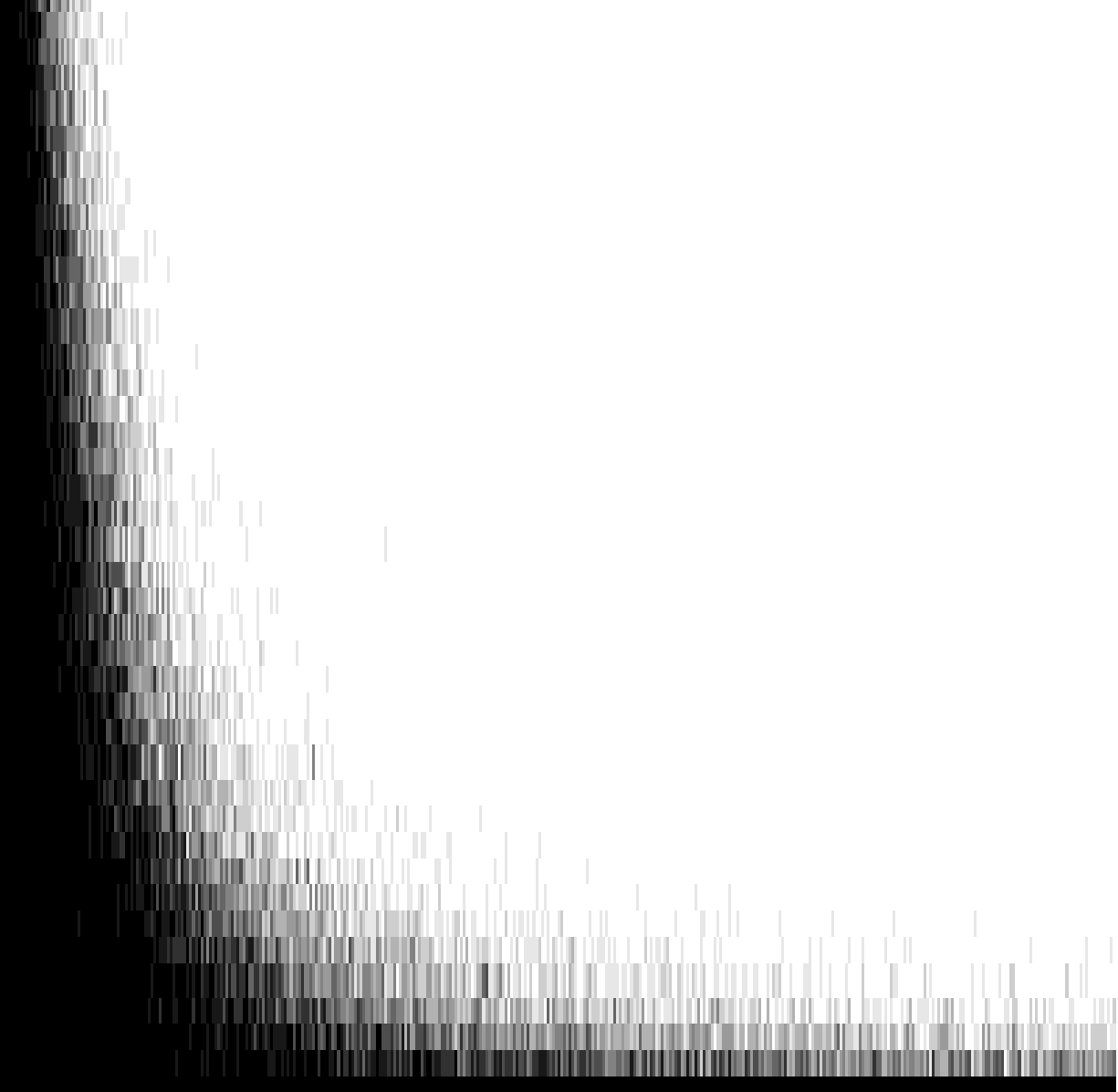}};
\draw[gray] (0,0) -- (6.135,0);
\draw[gray] (0,0) -- (0,6);
\draw[gray,very thin] (1.53375,0) -- (1.53375,0.1);
\draw[gray,very thin] (3.0675,0) -- (3.0675,0.1);
\draw[gray,very thin] (4.60125,0) -- (4.60125,0.1);
\draw[gray,very thin] (6.135,0) -- (6.135,0.1);
\draw[gray,very thin] (0,1.5) -- (0.1,1.5);
\draw[gray,very thin] (0,3) -- (0.1,3);
\draw[gray,very thin] (0,4.5) -- (0.1,4.5);
\draw[gray,very thin] (0,6) -- (0.1,6);
\draw (0,-0.2) node {\tiny{$0$}};
\draw (1.53375,-0.2) node {\tiny{$100$}};
\draw (3.0675,-0.2) node {\tiny{$200$}};
\draw (4.60125,-0.2) node {\tiny{$300$}};
\draw (6.135,-0.2) node {\tiny{$400$}};
\draw (3.0675,-0.75) node {$\zeta$: Distance between ellipsoid centers};
\draw (-0.15,0) node {\tiny{$0$}};
\draw (-0.2,1.5) node {\tiny{$10$}};
\draw (-0.2,3) node {\tiny{$20$}};
\draw (-0.2,4.5) node {\tiny{$30$}};
\draw (-0.2,6) node {\tiny{$40$}};
\draw (-0.75,3) node [rotate=90] {$M$: Rank of random projection};
\end{tikzpicture}
&
\begin{tikzpicture}
\node[inner sep=0pt] (plot) at (3.07,3) {\includegraphics[height=6cm]{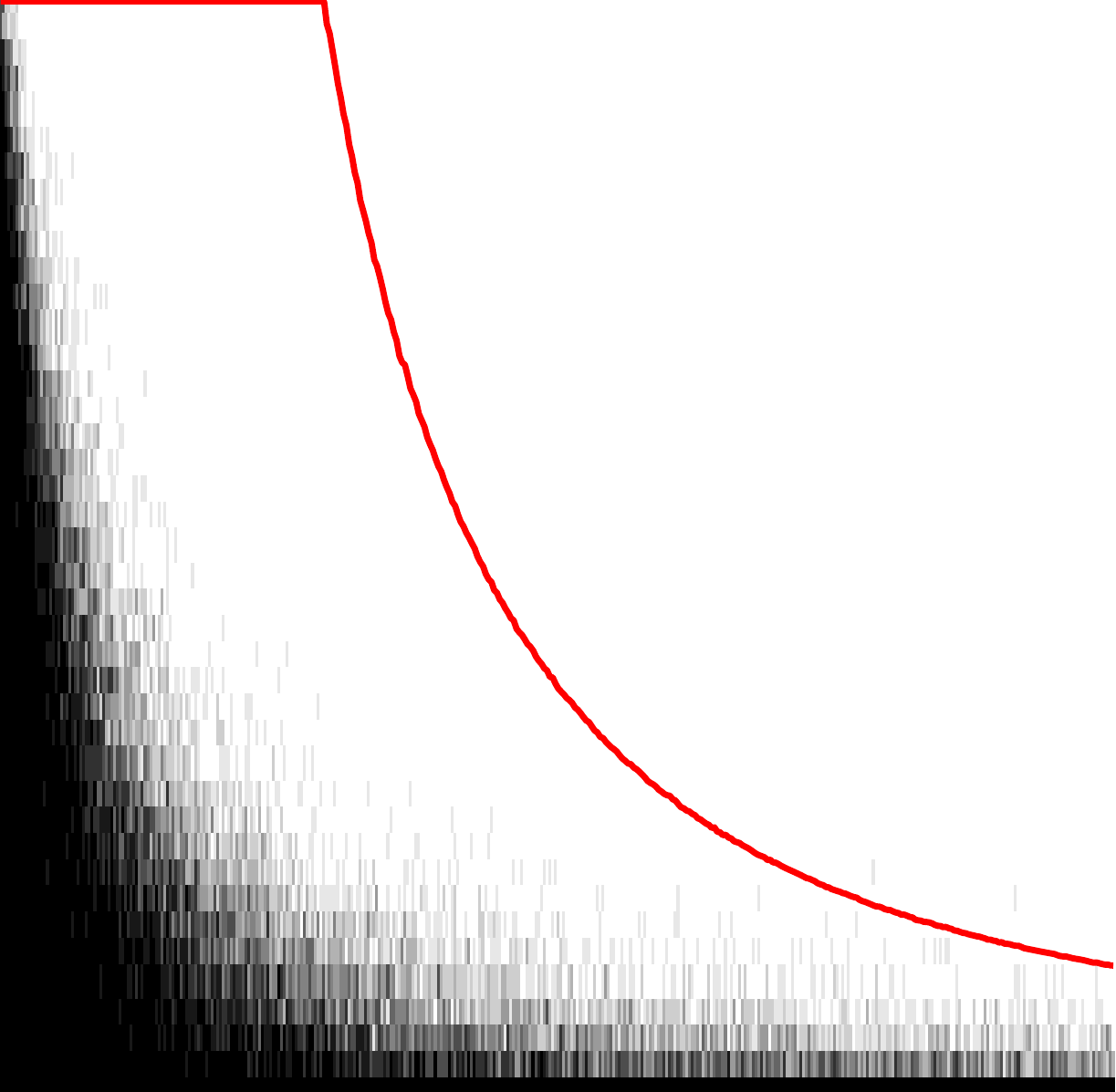}};
\draw[gray] (0,0) -- (6.135,0);
\draw[gray] (0,0) -- (0,6);
\draw[gray,very thin] (1.53375,0) -- (1.53375,0.1);
\draw[gray,very thin] (3.0675,0) -- (3.0675,0.1);
\draw[gray,very thin] (4.60125,0) -- (4.60125,0.1);
\draw[gray,very thin] (6.135,0) -- (6.135,0.1);
\draw[gray,very thin] (0,1.5) -- (0.1,1.5);
\draw[gray,very thin] (0,3) -- (0.1,3);
\draw[gray,very thin] (0,4.5) -- (0.1,4.5);
\draw[gray,very thin] (0,6) -- (0.1,6);
\draw (0,-0.2) node {\tiny{$0$}};
\draw (1.53375,-0.2) node {\tiny{$100$}};
\draw (3.0675,-0.2) node {\tiny{$200$}};
\draw (4.60125,-0.2) node {\tiny{$300$}};
\draw (6.135,-0.2) node {\tiny{$400$}};
\draw (3.0675,-0.75) node {$\zeta$: Distance between ellipsoid centers};
\draw (-0.15,0) node {\tiny{$0$}};
\draw (-0.2,1.5) node {\tiny{$10$}};
\draw (-0.2,3) node {\tiny{$20$}};
\draw (-0.2,4.5) node {\tiny{$30$}};
\draw (-0.2,6) node {\tiny{$40$}};
%\draw (-0.75,3) node [rotate=90] {$M$: Rank of random projection};
\end{tikzpicture}
\end{tabular}
\caption{
\label{figure_ellipsoids_phasetransition}
{\small 
Phase transition for a random projection to keep ellipsoids separated.
(a) Fixing the ambient dimension to be $N=40$, then for each $\zeta=1:400$ and $M=1:40$, we conducted $10$ trials.
For each trial, we randomly drew $A_1$ and $A_2$ as iid standard Wishart-distributed $N\times N$ matrices with $N$ degrees of freedom (i.e., $A_i=XX^\top$, where $X$ is $N\times N$ with iid $\mathcal{N}(0,1)$ entries), along with an $M\times N$ matrix $P$ with iid $\mathcal{N}(0,1)$ entries.
Plotted is the proportion of trials for which the ellipsoids are disjoint after applying $P$ (we did not record whether the ellipsoids were separated before projection).
For each of the 160,000 trials, the shape matrices satisfied $\zeta\leq\|A_1e\|_2+\|A_2e\|_2$, thereby rendering Theorem~\ref{theorem.two ellipsoids} irrelevant.
(b) Next, we performed the same experiment, except we changed the distribution of $A_1$ and $A_2$ so that $e$ is in the null space of both, and in the orthogonal complement of $e$, they are iid standard Wishart-distributed $(N-1)\times(N-1)$ matrices with $N-1$ degrees of freedom.
As such, the corresponding ellipsoids resided in parallel hyperplanes, and $\|A_1e\|_2+\|A_2e\|_2=0$ so that Theorem~\ref{theorem.two ellipsoids} applies.
For each trial, we stored the bound on $w_\cap$ from Theorem~\ref{theorem.two ellipsoids} and calculated the sample average of the squares of these bounds corresponding to each $\zeta=1:400$.
The red curve plots these sample averages (or 40, whichever is smaller)---think of this as an upper bound on the phase transition.
As one might expect, this bound appears to sharpen as the distance increases.
}\normalsize}
\end{center}
\end{figure}
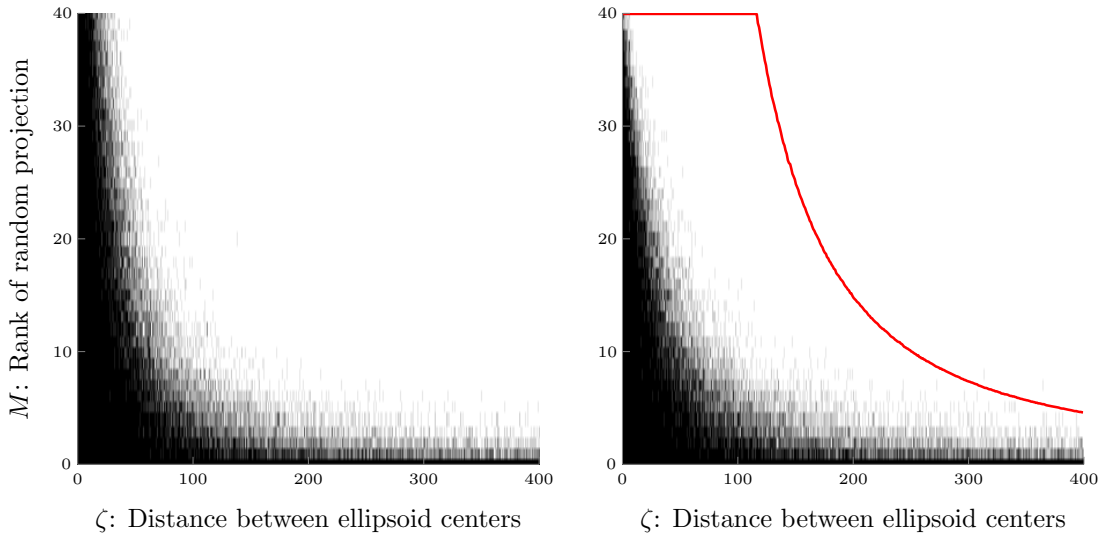

\subsection{The case of multiple convex sets}

Various classification tasks require one to distinguish between several different classes, and so one might ask for a random projection to maintain pairwise linear separability.
For a fixed projection dimension $M$, let $\eta_{ij}$ denote the probability that convex classes $S_i$ and $S_j$ collide after projection.
Then the union bound gives that the probability of maintaining separation is $\geq 1-\sum_{i,j:i<j}\eta_{ij}$.

This use of the union bound helps to illustrate the freedom which comes with a random projection.
Recall that Theorem~\ref{theorem.two ellipsoids} requires that projecting the ellipsoids onto the line spanned by the difference $c_1-c_2$ of their centers maintains separation.
In the case of multiple ellipsoids, one might then be inclined to project onto the span of $\{c_i-c_j\}_{i,j:i<j}$.
Generically, such a choice of projection puts $M=\binom{K}{2}=\Omega(K^2)$, where $K$ is the total number of classes.
On the other hand, suppose each pairwise distance $\|c_i-c_j\|$ is so large that the $(i,j)$th Gaussian width satisfies 
\begin{equation*}
w_\cap<\sqrt{2\log\bigg(\frac{1}{p}\binom{K}{2}\bigg)}.
\end{equation*}
Then by Corollary~\ref{corollary.Gordon}, taking $M=8\log(\binom{K}{2}/p)+1=O_p(\log K)$ ensures that classes $S_i$ and $S_j$ collide after projection with probability $\eta_{ij}\leq p/\binom{K}{2}$, and so the probability of maintaining overall separation is $\geq 1-p$.
Of course, we will not save so much in the projection dimension when the convex bodies are closer to each other, but we certainly expect $M<K^2$ in reasonable cases.

At this point, we note the similarity between the performance $M=O(\log K)$ and what the Johnson--Lindenstrauss lemma guarantees when the classes are each a single point.
Indeed, a random projection of $M=\Omega_\epsilon(\log K)$ dimensions suffices to ensure that pairwise distances are preserved to within a factor of $1\pm\epsilon$ with constant probability; this in turn ensures that pairwise separated points remain pairwise separated after projection.
In fact, the proof technique for the Johnson--Lindenstrauss lemma is similar:
First prove that a random projection typically preserves the norm of any vector, and then perform a union bound over all $\binom{K}{2}$ difference vectors.
One might be inspired to use Johnson--Lindenstrauss ideas to prove a result analogous to Theorem~\ref{theorem.two ellipsoids} (this was actually an initial attempt by the authors).
Unfortunately, since Johnson--Lindenstrauss does not account for the shape matrices $A_i$ of the ellipsoids, one is inclined to consider worst-case orientations, and so terms like $\|A_ie\|_2$ are replaced by spectral norms $\|A_i\|_2$ in the analysis, thereby producing a strictly weaker result.
Dasgupta~\cite{Dasgupta:99} uses this Johnson--Lindenstrauss approach to project a mixture of Gaussians while maintaining some notion of separation.

\section{Random projection versus principal component analysis}

In this section, we compare the performance of random projection and principal component analysis (PCA) for dimensionality reduction.
First, we should briefly review how to perform PCA.
Consider a collection of data points $\{x_i\}_{i=1}^p\subseteq\mathbb{R}^N$, and define the empirical mean by $\bar{x}:=\frac{1}{p}\sum_{i=1}^px_i$.
Next, consider the empirical inertia matrix
\begin{equation*}
\widehat{\Sigma}
:=\frac{1}{p}\sum_{i=1}^p (x_i-\overline{x})(x_i-\bar{x})^\top 
=\frac{1}{p}\sum_{i=1}^p x_ix_i^\top-\bar{x}\bar{x}^\top.
\end{equation*}
The eigenvectors of $\widehat{\Sigma}$ with the largest eigenvalues are identified as the \textit{principal components}, and the idea of PCA is to project $\{x_i\}_{i=1}^p$ onto the span of these components for dimensionality reduction.

In this section, we will compare random projection with PCA in a couple of ways.
First, we observe some toy examples of data sets that illustrate when PCA is better, and when random projection is better.
Later, we make a comparison using a real-world hyperspectral data set.

\subsection{Comparison using toy examples}

Here, we consider a couple of extreme data sets which illustrate when PCA outperforms random projection and vice versa.
Our overarching model for the data sets will be the following:
Given a collection of disjoint balls $\{S_i\}_{i=1}^K$ in $\mathbb{R}^N$, we independently draw $p$ data points uniformly from $S:=\bigcup_{i=1}^K S_i$.
When $p$ is large, we can expect $\widehat{\Sigma}$ to be very close to
\begin{equation*}
\Sigma:=\frac{1}{\operatorname{vol}(S)}\sum_{i=1}^K\int_{S_i}xx^\top dx-\mu\mu^\top
\end{equation*}
by the law of large numbers; here, $\mu\in\mathbb{R}^N$ denotes the mean of the distribution.
Recall that the projection dimension for PCA is the number of large eigenvalues of $\widehat{\Sigma}$.
Since the operator spectrum is a continuous function of the operator, we can count large eigenvalues of $\Sigma$ to estimate this projection dimension.
The following lemma will be useful to this end:

\begin{lemma}
\label{lemma.pca}
Consider a ball of the form $S:=c+r\mathcal{B}$, where $\mathcal{B}\subseteq\mathbb{R}^N$ denotes the ball centered at $0$ of radius $1$.
Define the operator
\begin{equation*}
W:=\int_S xx^\top dx.
\end{equation*}
Then the span of $c$ and its orthogonal complement form the eigenspaces of $W$ with eigenvalues 
\begin{equation*}
\lambda_c
=r^N\|c\|^2\operatorname{vol}(\mathcal{B})+Cr^{N+2},
\qquad
\lambda_{c^\perp}
=Cr^{N+2},
\end{equation*}
respectively, where $C$ is some constant depending on $N$.
\end{lemma}

\begin{proof}
Pick any vector $v\in\mathbb{R}^N$ of unit norm.
Then
\begin{equation*}
v^\top W v
=\int_\mathcal{B} v^\top(c+ry)(c+ry)^\top v r^N dy
=(v^\top c)^2\cdot r^N\operatorname{vol}(\mathcal{B})+r^{N+2}v^\top\bigg(\int_\mathcal{B} yy^\top dy\bigg)v.
\end{equation*}
Notice that the operator $\int_\mathcal{B} yy^\top dy$ is invariant under conjugation by any rotation matrix.
As such, this operator is a constant $C$ multiple of the identity operator.
Thus, $v^\top W v$ is maximized at $\lambda_c$ when $v$ is a normalized version of $c$, and minimized at $\lambda_{c^\perp}$ whenever $v$ is orthogonal to $c$.
\end{proof}

We start by considering the case where $S$ is composed of two balls, namely $S_1:=c+r\mathcal{B}$ and $S_2:=-c+r\mathcal{B}$.
As far as random projection is concerned, in this case, we are very familiar with the required projection dimension: $\Omega_\eta(Nr^2/\|c\|^2)$.
In particular, as $\|c\|$ approaches $r$, a random projection cannot provide much dimensionality reduction.
To compare with PCA, note that in this case, $\Sigma$ is a scalar multiple of $W_1+W_2$, where 
\begin{equation*}
W_i:=\int_{S_i} xx^\top dx.
\end{equation*}
Moreover, it is easy to show that $W_1=W_2$.
By Lemma~\ref{lemma.pca}, the dominant eigenvector of $W_i$ is $c$, and so PCA would suggest to project onto the one-dimensional subspace spanned by $c$.
Indeed, this projection always preserves separation, and so in this case, PCA provides a remarkable savings in projection dimension.

Now consider the case where $S$ is composed of $2N$ balls $\{S_{n,1}\}_{n=1}^N\cup\{S_{n,2}\}_{n=1}^N$ defined by $S_{n,1}:=e_n+r\mathcal{B}$ and $S_{n,2}:=-e_n+r\mathcal{B}$, where $e_n$ denotes the $n$th identity basis element.
Then $\Sigma$ is a scalar multiple of $\sum_{n=1}^N(W_{n,1}+W_{n,2})$, where
\begin{equation*}
W_{n,i}:=\int_{S_{n,i}} xx^\top dx.
\end{equation*}
Recall that $W_{n,1}=W_{n,2}$.
Then $\Sigma$ is simply a scalar multiple of $\sum_{n=1}^NW_{n,1}$.
By Lemma~\ref{lemma.pca}, the $W_{n,1}$'s are all diagonal, and their diagonals are translates of each other.
As such, their sum (and therefore $\Sigma$) is a scalar multiple of the identity matrix---in this case, PCA would choose to not project down to fewer dimensions.
On the other hand, if we take
\begin{equation*}
M>\left(\sqrt{N\bigg(\frac{2r}{\sqrt{2}}\bigg)^2+1}+\sqrt{2\log\bigg(\frac{1}{p}\binom{2N}{2}\bigg)}\right)^2+1,
\end{equation*}
then by Corollary~\ref{corollary.Gordon}, a random projection maintains separation between any fixed pair of balls from $\{S_{n,1}\}_{n=1}^N\cup\{S_{n,2}\}_{n=1}^N$ with probability $\geq 1-p/\binom{2N}{2}$, and so by the union bound, the balls are pairwise separated with probability $\geq1-p$.
In particular, if $r=O(N^{-1/2})$, then we can take $M=O_p(\log N)$.

Overall, random projection performs poorly when the classes are close, but when there are multiple sufficiently separated classes, you can expect a dramatic dimensionality reduction.
As for PCA, we have constructed a toy example for which PCA performs well (the case of two balls), but in general, the performance of PCA seems difficult to describe theoretically.
Whereas the performance of random projection can be expressed in terms of ``local'' conditions (e.g., pairwise separation), as the last example illustrates, the performance of PCA can be dictated by more ``global'' conditions (e.g., the geometric configuration of classes).
In the absence of theoretical guarantees for PCA, the following subsection provides simulations with real-world hyperspectral data to illustrate its performance compared to random projection.

\subsection{Simulations with hyperspectral data}

One specific application of dimensionality reduction is the classification of hyperspectral data.
For this application, the idea is to distinguish materials by observing them across hundreds of spectral bands (like the red, green and blue bands that the human eye detects).
Each pixel of a hyperspectral image can be viewed as a vector of spectral information, capturing how much light of various frequencies is being reradiated from that portion of the scene. 
A hyperspectral image is naturally represented as a data cube with two spatial indices and one spectral index, and a common task is to identify the material observed at each pair of spatial indices.
To do this, one might apply \textit{per-pixel classification}, in which a classifier simply identifies the material in a given pixel from its spectral content, ignoring any spatial context.
Since the spectral information is high-dimensional, it is natural to attempt dimensionality reduction before classification.
A popular choice for this task is PCA~\cite{Johnson:08,RobilaM:06}, and in this subsection, we provide some preliminary simulations to compare its performance with random projection.

All experiments described in this subsection were conducted using the Indian Pines hyperspectral data set~\cite{HRSS:online}. 
This data set consists of a hyperspectral image with $145\times 145$ pixels, each containing spectral reflectance data represented by a vector of length $N=200$.
Each pixel corresponds to a particular type of vegetation or crop, such as corn or wheat, with a total of $17$ different classes (see Figure~\ref{figureIndianPines} for an illustration).

\begin{figure}[t]
\begin{center}
\includegraphics[width=0.7\textwidth]{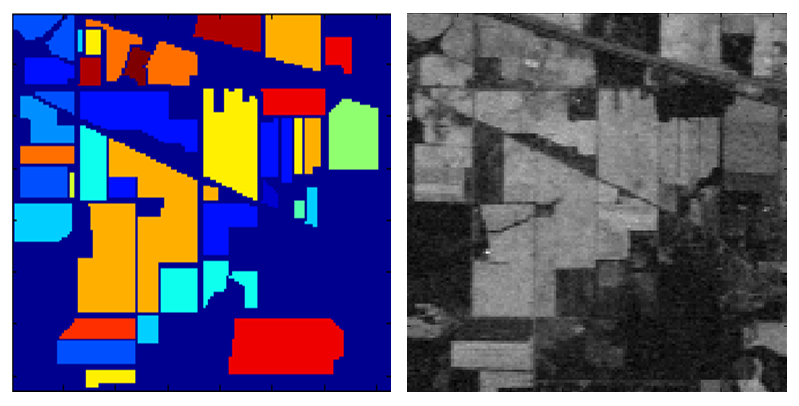}
\caption{
\label{figureIndianPines}
{\small 
The Indian Pines hyperspectral data set~\cite{HRSS:online}.
Each pixel corresponds to a different type of vegetation or crop.
The ground truth image of labels is depicted on the left, and a sample band of the data set is displayed on the right.
}\normalsize}
\end{center}
\end{figure}

For our simulations, the task will consist of using the known labels of a training set (a small subset of the $21,025 = 145\times 145$ pixels) to make accurate predictions for the remainder of the pixels.
To keep the simulations fast, each simulation considers a small patch of pixels.
More precisely, given a patch of $p$ pixels and a prescribed training ratio $r$, we pick a random subset of the pixels of size $rp$ to be the training set.
We use the labels from this training set to train a classifier that will then attempt to guess the label of each of the other $(1-r)p$ pixels from the location of its spectral reflectance in $200$-dimensional space.
The classifier we use is MATLAB's built-in implementation of multinomial logistic regression.
Performance is measured by classification error and runtime.

Given this setting, for different values of projection dimension $M$, we draw an $M\times N$ matrix $P$ with iid $\mathcal{N}(0,1)$ entries and replace every spectral reflectance data point $x$ by $Px$.
In the degenerate case $M=N$, we simply take $P$ to be the identity matrix.
For comparison, we also use principal component analysis (PCA) for dimensionality reduction, which will interrogate the training set to identify $M$ principal components before projecting the data set onto the span of these components.
An immediate advantage of random projection is that it allows the sensing mechanism to blindly compress the data, as it does not need a training set to determine the compression function.

Figure~\ref{figureIndianPines2_trainningratio0.5and0.2} uses different patches of the Indian Pines dataset and different training ratios to compare both the classification accuracy and runtime of multinomial logistic regression when applied to various projections of the data set.
The first experiment focuses on a small patch of $225$ pixels, and the second considers a patch of $3481$ pixels.
These experiments reveal a few interesting phenomena.
First of all, dimensionality reduction leads to impressive speedups in runtime.
Perhaps more surprising is the fact that there seems to be an improvement in classification performance after projecting the data.
We are far from completely understanding this behavior, but we suspect it has to do with regularization and overfitting.

\begin{figure}
\begin{center}
\includegraphics[trim=4cm 0cm 4cm 0cm, width=0.8\textwidth, height=0.4\textheight]{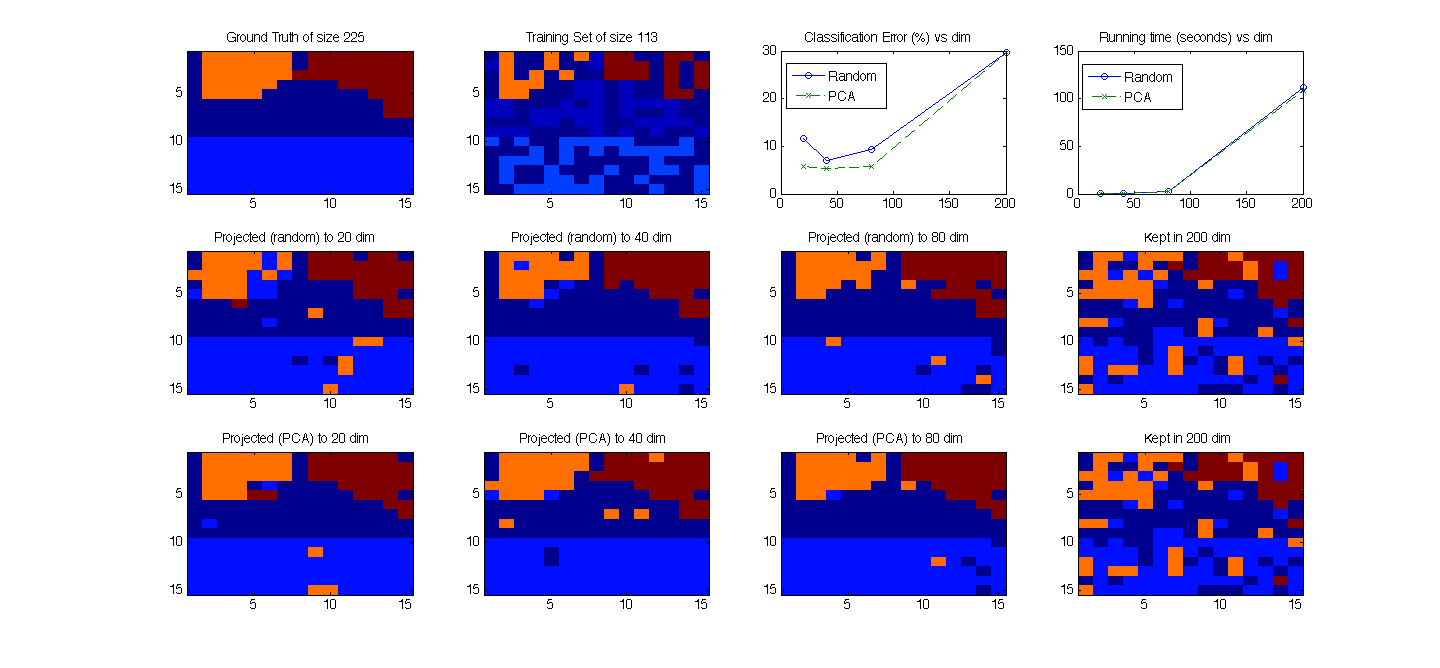}
\includegraphics[trim=4cm 0cm 4cm 0cm, width=0.8\textwidth, height=0.4\textheight]{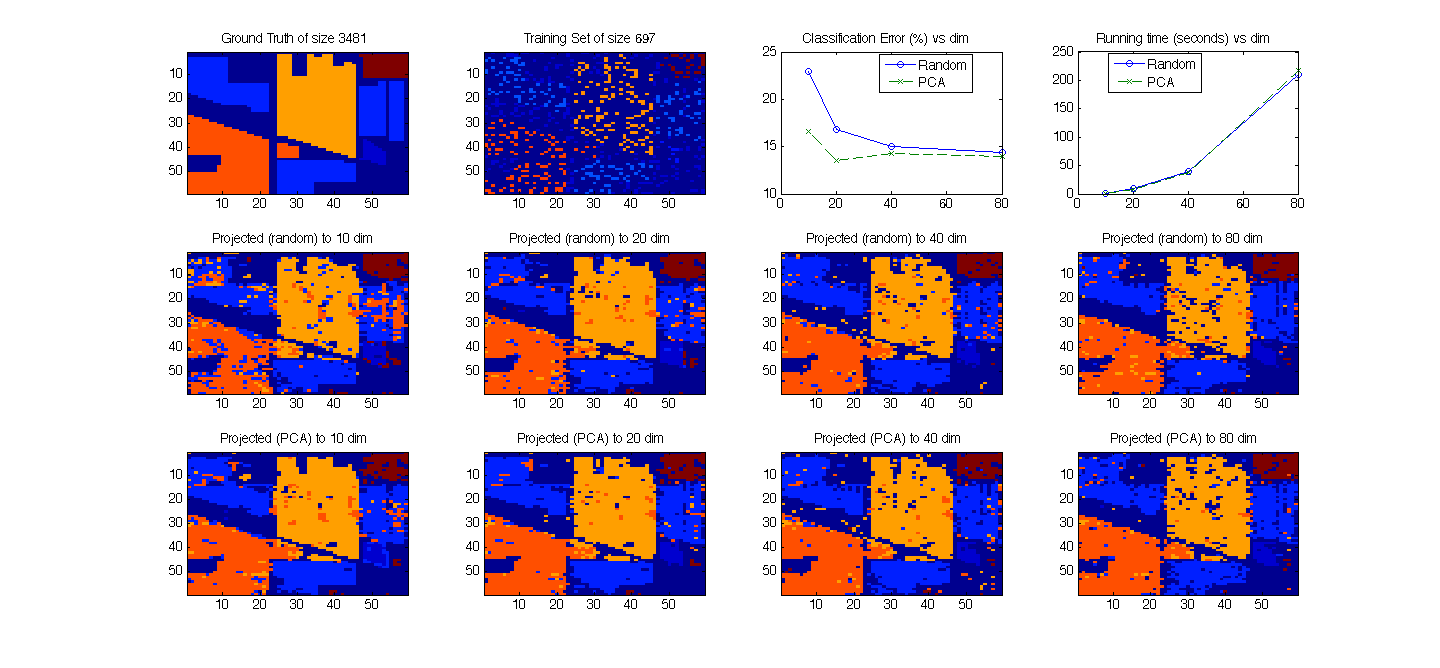}
\caption{
\label{figureIndianPines2_trainningratio0.5and0.2}
{\small 
The performance of classification by multinomial logistic regression after projecting onto subspaces of various dimensions $M$.
Depicted are two particular patches of the entire Indian Pines data set---the top uses a patch of $225$ pixels, while the bottom uses a patch of $3481$ pixels.
In each case, the first two plots in the first row depict the ground truth labels in the patch, as well as the random training set we selected.
The third plot compares, for different values of projection dimension $M$, the classification error incurred with random projection and with principal component analysis. 
The fourth plot shows the runtime (in seconds) for different values of $M$.
The second and third rows depict the classification outcomes when using random projection and PCA, respectively.
One can see that dimensionality reduction not only speeds up the algorithm, but also improves the classification performance by discouraging overfitting.
}\normalsize}
\end{center}
\end{figure}

It is also interesting how similar random projection and PCA perform.
Note that the PCA method has an unfair advantage since it is data-adaptive, meaning that it uses the training data to select the projection, and in practical applications for which the sensing process is expensive, one might be interested in projecting in a non-adaptive way, thereby allowing for less sensing.
Our simulations suggest that the expense is unnecessary, as a random projection will provide almost identical performance.
As indicated in the previous subsection, random projection is also better understood as a means to maintain linear separability, and so there seems to be little benefit in choosing PCA over random projection (at least for this sort of classification task).

\section{Future work}

One of the main points of this paper is that random projections can maintain separation between sufficiently separated sets, and this is useful for classification in the projected domain.
Given the mindset of compressed sensing, it is impressive that the sets need not be low-dimensional to enjoy separation in the projected domain.
What this suggests is a more general notion of simplicity that is at play, of which low-dimensionality and sufficient separation are mere instances.
Obviously, understanding this general notion is a worthy subject of future work.

From a more applied perspective, it would be worth investigating alternative notions of distortion.
Indeed, linear separability is the best-case scenario for classification, but it is not at all necessary.
After identifying any worthy notion of distortion, one might study how much distortion is incurred by random projection, and hopefully some of the ideas contained in this paper will help.

One of our main results (Theorem~\ref{theorem.two ellipsoids}) provides a sufficient number of rows for a random projection to maintain separation between ellipsoids.
However, as illustrated in Figure~\ref{figure_ellipsoids_phasetransition}, this bound is far from optimal.
Considering this case of two ellipsoids is somewhat representative of the more general case of two convex sets (as we identified using Theorem~\ref{theorem.concentration of volume}), improvements to Theorem~\ref{theorem.two ellipsoids} would be rather interesting.
In particular, it would be nice to characterize the phase transition in terms of the ellipsoids' parameters, as we already have in the case of two balls.

Finally, the random projections we consider here all have iid $\mathcal{N}(0,1)$ entries, but real-world sensing systems may not enjoy this sort of flexibility.
As such, it would be interesting to extend the results of this paper to more general classes of random projections, in particular, random projections which can be implemented with a hyperspectral imager (say).

\section{Appendix: Proofs}

\subsection{Proof of Gordon's Escape Through a Mesh Theorem}

This proof is chiefly based on the following result, which appears as Corollary~1.2 in~\cite{Gordon:88}:

\begin{gordons-comparison-theorem}
Let $S$ be a closed subset of $\mathbb{S}^{n-1}$.
Draw an $M\times N$ matrix $P$ with iid $\mathcal{N}(0,1)$ entries.
Then
\begin{equation*}
\mathbb{E}\bigg[\min_{x\in S}\|Px\|_2\bigg]
\geq\lambda_M-w(S),
\end{equation*}
where $\lambda_M:=\mathbb{E}\|g\|_2$, and $g$ is a random $M$-dimensional vector with iid $\mathcal{N}(0,1)$ entries.
\end{gordons-comparison-theorem}

To prove the escape theorem, consider the function
\begin{equation*}
f_S\colon P\mapsto\min_{x\in S}\|Px\|_2.
\end{equation*}
Gordon's Comparison Theorem gives that $\mathbb{E}[f_S]\geq\lambda_M-w(S)$, and so
\begin{align}
\operatorname{Pr}\Big(Y\cap S=\emptyset\Big)
\nonumber
&=\operatorname{Pr}\Big(\min_{x\in S}\|Px\|_2>0\Big)\\
\nonumber
&=\operatorname{Pr}\Big(\min_{x\in S}\|Px\|_2>\big(\lambda_M-w(S)\big)-\big(\lambda_M-w(S)\big)\Big)\\
\label{eq.escape to bound}
&\geq\operatorname{Pr}\Big(\min_{x\in S}\|Px\|_2>\mathbb{E}[f_S]-\big(\lambda_M-w(S)\big)\Big).
\end{align}
Next, we note that $f_S$ is Lipschitz with respect to the Frobenius norm with constant $1$, and so we can appeal to~(1.6) of~\cite{LedouxT:91} to get
\begin{equation}
\label{eq.lipschitz bound}
\operatorname{Pr}\Big(f_S(P)>\mathbb{E}[f_S]-t\Big)\geq 1-e^{-t^2/2}
\qquad
\forall t>0.
\end{equation}
Taking $t=\lambda_M-w(S)$ and applying \eqref{eq.lipschitz bound} to \eqref{eq.escape to bound} then gives the result.

\subsection{Proof of Lemma~\ref{lemma.circcone lower bound}}

Let $\mathcal{C}_-$ denote the cone generated by the Minkowski difference $S_1-S_2$.
We will show $\mathcal{C}_-=\operatorname{Circ}(\alpha)$ by verifying both containments.

We begin by finding the smallest $\alpha\in[0,\pi/2]$ for which $\mathcal{C}_-\subseteq\operatorname{Circ}(\alpha)$.
By the definition of $\operatorname{Circ}(\alpha)$, this containment is equivalent to
\begin{equation}
\label{eq.alpha requirement}
\cos\alpha
\leq\inf_{z\in \mathcal{C}_-}\frac{\langle z,c_1-c_2\rangle}{\|z\|\|c_1-c_2\|}
=\min_{z\in S_1-S_2}\frac{\langle z,c_1-c_2\rangle}{\|z\|\|c_1-c_2\|}.
\end{equation}
To find the smallest such $\alpha$, we solve this optimization problem.
Taking $d:=c_1-c_2$, then $S_1-S_2=(r_1+r_2)\mathcal{B}+d$, and so we seek to
\begin{equation*}
\mbox{minimize}
\quad
f(y)=\frac{\langle y+d,d\rangle}{\|y+d\|\|d\|}
\quad
\mbox{subject to}
\quad
\|y\|\leq r_1+r_2.
\end{equation*}
Quickly note that the objective function is well defined over the feasibility region due to the assumption $r_1+r_2<\|d\|$.
We first claim that $f(y)$ is minimized on the boundary, i.e., where $\|y\|=r_1+r_2$.
To see this, suppose $\|y\|<r_1+r_2$, and letting $P_{d^\perp}$ denote the orthogonal projection onto the orthogonal complement of the span of $d$, take $t>0$ such that $\|y+tP_{d^\perp}y\|=r_1+r_2$.
Then $y+tP_{d^\perp}y$ lies on the boundary and 
\begin{equation*}
f(y+tP_{d^\perp}y)
=\frac{\langle y+tP_{d^\perp}y+d,d\rangle}{\|y+tP_{d^\perp}y+d\|\|d\|}
=\frac{\langle y+d,d\rangle}{\|y+tP_{d^\perp}y+d\|\|d\|}
<\frac{\langle y+d,d\rangle}{\|y+d\|\|d\|}
=f(y).
\end{equation*}
As such, it suffices to minimize subject to $\|y\|=r_1+r_2$.
At this point, the theory of Lagrange multipliers can be applied since the equality constraint $g(y):=\|y\|^2=(r_1+r_2)^2$ is a level set of a function whose gradient $\nabla g(y)=2y$ does not vanish on the level set.
Thus, the minimizers of $f$ with $g(y)=(r_1+r_2)^2$ satisfy $\nabla f(y)=-\lambda\nabla g(y)=-2\lambda y$ for some Lagrange multiplier $\lambda\in\mathbb{R}$.

To continue, we calculate $\nabla f(y)$.
It is actually easier to calculate the gradient of $h(u):=\langle u,d\rangle/\|u\|\|d\|$:
\begin{equation*}
\nabla h(u)
=\frac{1}{\|u\|^2}\bigg(d-\bigg\langle\frac{u}{\|u\|},d\bigg\rangle\frac{u}{\|u\|}\bigg).
\end{equation*}
Note that $\nabla h(u)=0$ only if $u$ is a nontrivial multiple of $d$, i.e., only if $u$ maximizes $h$ (by Cauchy--Schwarz).
Also, it is easy to verify that $\langle u,\nabla h(u)\rangle=0$.
Overall, changing variables $u\leftarrow y+d$ gives that any minimizer $y^\natural$ of $f$ subject to $\|y\|=r_1+r_2$ satisfies
\begin{align}
\label{eq.tight a}
\nabla f(y^\natural)
&=-2\lambda y^\natural\qquad \mbox{for some }\lambda\in\mathbb{R},\\
\label{eq.tight b}
\nabla f(y^\natural)
&\neq 0,\\
\label{eq.tight c}
\langle y^\natural+d,\nabla f(y^\natural)\rangle
&=0.
\end{align}
At this point, \eqref{eq.tight a} and \eqref{eq.tight b} together imply that $\nabla f(y^\natural)$ is a nontrivial multiple of $y^\natural$, and so combining with \eqref{eq.tight c} gives
\begin{equation*}
\langle y^\natural+d,y^\natural\rangle=0.
\end{equation*}
As such, $0$, $d$ and $y^\natural+d$ form vertices of a right triangle with hypotenuse $\|d\|$, and the smallest $\alpha$ satisfying \eqref{eq.alpha requirement} is the angle between $d$ and $y^\natural+d$.
Thus, $\sin\alpha=\|y^\natural\|/\|d\|=(r_1+r_2)/\|c_1-c_2\|$.

It remains to prove the reverse containment, $\operatorname{Circ}(\alpha)\subseteq\mathcal{C}_-$, for this particular choice of $\alpha$.
Define
\begin{equation*}
G
:=\{z:\langle z,d\rangle=\|z\|\|d\|\cos\alpha,~\|z\|=\|y^\natural+d\|\}.
\end{equation*}
Then $\operatorname{Circ}(\alpha)$ is the cone generated by $G$, and so it suffices to show that $G\subseteq S_1-S_2=(r_1+r_2)\mathcal{B}+d$.
To this end, pick any $z\in G$, and consider the triangle with vertices $0$, $d$ and $z$.
By definition, the angle between $d$ and $z$ is $\alpha$, and the side $z$ has length $\|y^\natural+d\|$.
As such, by the side-angle-side postulate, this triangle is congruent to the triangle with vertices at $0$, $d$ and $y^\natural+d$.
This implies that the side between $z$ and $d$ has length $\|z-d\|=\|y^\natural\|=r_1+r_2$, and so $z=(z-d)+d\in(r_1+r_2)\mathcal{B}+d$, as desired.

\subsection{Proof of Theorem~\ref{theorem.two ellipsoids}}

This proof makes use of the following lemma:

\begin{lemma}
\label{lemma.applying Jensens twice}
Take an $n\times n$ matrix $A$ and let $g$ have iid $\mathcal{N}(0,1)$ entries.
Then
\begin{equation*}
\sqrt{\frac2{\pi}}\|A\|_F \leq \mathbb{E}\|Ag\|_2 \leq \|A\|_F.
\end{equation*}
\end{lemma}

\begin{proof}
Let $A=UDV$ be the singular value decomposition of $A$.
Since the Gaussian is isotropic, $\mathbb{E}\|Ag\|_2=\mathbb{E}\|Dg\|_2$, and since the function $x\mapsto x^2$ is convex, Jensen's inequality gives
\begin{equation*}
\mathbb{E}\|Dg\|_2
\leq\sqrt{ \mathbb{E}\|Dg\|_2^2}
=\sqrt{\sum_{i=1}^n D_{ii}^2 \mathbb{E}g_i^2}
=\|D\|_F
=\|A\|_F.
\end{equation*}
Similarly, since $x\mapsto \|x\|_2$ is convex, we can also use Jensen's inequality to get
\begin{equation*}
\mathbb{E}\|Dg\|_2
=\mathbb{E}\sqrt{\sum_{i=1}^n D_{ii}^2 g_i^2}
\geq\sqrt{\sum_{i=1}^n \left(\mathbb{E}|D_{ii} g_i|\right)^2}
=\mathbb{E}|g_1|\sqrt{\sum_{i=1}^n D_{ii}^2}
=\sqrt{\frac2{\pi}}\|A\|_F,
\end{equation*}
which completes the proof.
\end{proof}

To prove Theorem~\ref{theorem.two ellipsoids}, let $\mathcal{C}_-$ denote the cone generated by the Minkowski difference $S_1-S_2$.
We will exploit Proposition~\ref{proposition 3.6}, which gives the following estimate in terms of the polar cone $\mathcal{C}_-^*:=\{w:\langle w,z\rangle\leq 0 ~ \forall z\in \mathcal{C}_-\}$:
\begin{equation*}
w_\cap
\leq\mathbb{E}_g\Big[\|g-\Pi_{\mathcal{C}_-^*}(g)\|_2\Big],
\end{equation*}
where $g$ has iid $\mathcal{N}(0,1)$ entries and $\Pi_{\mathcal{C}_-^*}$ denotes the Euclidean projection onto $\mathcal{C}_-^*$.
Instead of directly computing the distance between $g$ and its projection onto $\mathcal{C}_-^*$, we will construct a mapping $\widetilde{\Pi}$ which sends $g$ to some member of $\mathcal{C}_-^*$, but for which distances are easier compute; indeed $\|g-\widetilde{\Pi}(g)\|_2$ will be an upper bound on $\|g-\Pi_{\mathcal{C}_-^*}(g)\|_2$.
Consider the polar decomposition $c_2-c_1=\zeta e$, where $\zeta>0$.
Then we can decompose $g=g_1e+g_2$, and we define $\widetilde{\Pi}(g)$ to be the point in $\mathcal{C}_-^*$ of the form $\alpha e+g_2$ which is closest to $g$.
With this definition, we have 
\begin{equation*}
\|g-\Pi_{\mathcal{C}_-^*}(g)\|_2
\leq\|g-\widetilde{\Pi}(g)\|_2
=\min |g_1-\alpha|
\quad
\mbox{s.t.}
\quad
\alpha e+g_2\in \mathcal{C}_-^*.
\end{equation*}
To simplify this constraint, we find a convenient representation of the polar cone:
\begin{align*}
\mathcal{C}_-^*
&=\{w:\langle w,z\rangle\leq 0 ~ \forall z\in \mathcal{C}_-\}\\
&=\{w:\langle w,u-v\rangle\leq 0 ~ \forall u\in S_1, v\in S_2\}\\
&=\{w:\langle w,c_2-c_1\rangle\geq\langle w,A_1x\rangle-\langle w,A_2y\rangle ~ \forall x,y\in \mathcal{B}\}\\
&=\Big\{w:\langle w,c_2-c_1\rangle\geq\max_{x\in \mathcal{B}}\langle w,A_1x\rangle+\max_{y\in \mathcal{B}}\langle w,-A_2y\rangle\Big\}\\
&=\Big\{w:\langle w,c_2-c_1\rangle\geq\max_{x\in \mathcal{B}}\langle A_1^\top w,x\rangle+\max_{y\in \mathcal{B}}\langle -A_2^\top w,y\rangle\Big\}\\
&=\{w:\langle w,c_2-c_1\rangle\geq \|A_1w\|_2+\|A_2w\|_2\},
\end{align*}
where the last step uses the fact that each $A_i$ is symmetric.
The constraint $\alpha e+g_2\in \mathcal{C}_-^*$ is then equivalent to
\begin{equation*}
\alpha\zeta
\geq\|A_1(\alpha e+g_2)\|_2+\|A_2(\alpha e+g_2)\|_2.
\end{equation*}
At this point, we focus on the case in which the projection $e^\top S_1$ is disjoint from $e^\top S_2$.
In this case, we have the following strict inequality:
\begin{equation*}
\max_{x\in \mathcal{B}}\langle c_1+A_1x,e\rangle
=\max_{u\in S_1}\langle u,e\rangle
<\min_{v\in S_2}\langle v,e\rangle
=\min_{y\in \mathcal{B}}\langle c_2+A_2y,e\rangle,
\end{equation*}
and rearranging then gives
\begin{align*}
\zeta
=\langle c_2-c_1,e\rangle
&>\max_{x\in \mathcal{B}}\langle A_1x,e\rangle+\max_{y\in \mathcal{B}}\langle -A_2x,e\rangle\\
&=\max_{x\in \mathcal{B}}\langle x,A_1^\top e\rangle+\max_{y\in \mathcal{B}}\langle x,-A_2^\top e\rangle
=\|A_1e\|_2+\|A_2e\|_2.
\end{align*}
As such, taking
\begin{equation}
\label{eq.defn of alpha star}
\alpha
\geq \alpha^*
:=\frac{\|A_1g_2\|_2+\|A_2g_2\|_2}{\zeta-\big(\|A_1e\|_2+\|A_2e\|_2\big)}
\end{equation}
produces a point $\alpha e+g_2 \in \mathcal{C}_-^*$, considering
\begin{equation*}
\alpha\zeta
\geq \alpha\big(\|A_1e\|_2+\|A_2e\|_2\big)+\|A_1g_2\|_2+\|A_2g_2\|_2
\geq \|A_1(\alpha e+g_2)\|_2+\|A_2(\alpha e+g_2)\|_2,
\end{equation*}
where the last step follows from the triangle inequality.
Note that if $g_1\geq \alpha^*$, then we can take $\alpha=g_1$ to get $\|g-\widetilde{\Pi}(g)\|_2=0$.
Otherwise, $\|g-\widetilde{\Pi}(g)\|_2\leq|g_1-\alpha^*|=\alpha^*-g_1$.
Overall, we have
\begin{equation*}
\|g-\Pi_{\mathcal{C}_-^*}(g)\|_2
\leq\|g-\widetilde{\Pi}(g)\|_2
\leq(\alpha^*-g_1)_+.
\end{equation*}
By the monotonicity of expectation, we then have
\begin{equation}
\label{eq.bound on w cap 1}
w_\cap
\leq\mathbb{E}_g\Big[\|g-\Pi_{\mathcal{C}_-^*}(g)\|_2\Big]
\leq\mathbb{E}_g(\alpha^*-g_1)_+
=\mathbb{E}_{g_2}\Big[\mathbb{E}_{g_1}\Big[(\alpha^*-g_1)_+\Big|g_2\Big]\Big].
\end{equation}
To estimate the right-hand side, we first have
\begin{equation}
\label{eq.bound on w cap 2}
\mathbb{E}_{g_1}\Big[(\alpha^*-g_1)_+\Big|g_2\Big]
=\int_{-\infty}^\infty(\alpha^*-z)_+d\Phi(z)
=\alpha^*\Phi(\alpha^*)+\frac{1}{\sqrt{2\pi}}e^{-(\alpha^*)^2/2},
\end{equation}
which lies between $\alpha^*/2$ and $\alpha^*+1/\sqrt{2\pi}$ since $\alpha\geq0$.

Let $P_{e^\perp}$ denote the $n\times n$ orthogonal projection onto the orthogonal complement of the span of $e$.
Appealing to Lemma~\ref{lemma.applying Jensens twice} with $A:=A_iP_{e^\perp}$ then gives
\begin{equation*}
\mathbb{E}\|A_ig_2\|_2
=\mathbb{E}\|A_iP_{e^\perp}g\|_2
\leq\|A_iP_{e^\perp}\|_F
\leq\|A_i\|_F,
\end{equation*}
where the last inequality follows from the fact that each row of $A_iP_{e^\perp}$ is a projection of the corresponding row in $A_i$, and therefore has a smaller $2$-norm.
Considering \eqref{eq.defn of alpha star}, this implies
\begin{equation*}
\mathbb{E}_{g_2}\alpha^*
\leq \frac{\|A_1\|_F+\|A_2\|_F}{\zeta-\big(\|A_1e\|_2+\|A_2e\|_2\big)},
\end{equation*}
which combined with \eqref{eq.bound on w cap 1} and \eqref{eq.bound on w cap 2} then gives
\begin{equation*}
w_\cap
\leq\mathbb{E}_{g_2}\Big[\mathbb{E}_{g_1}\Big[(\alpha^*-g_1)_+\Big|g_2\Big]\Big]
\leq\frac{\|A_1\|_F+\|A_2\|_F}{\zeta-\big(\|A_1e\|_2+\|A_2e\|_2\big)}+\frac{1}{\sqrt{2\pi}}.
\end{equation*}

\subsection*{Acknowledgments}
The authors thank Matthew Fickus and Katya Scheinberg for insightful discussions.
A.\ S.\ Bandeira was supported by AFOSR award FA9550-12-1-0317, D.\ G.\ Mixon was supported by NSF award DMS-1321779, and B.\ Recht was supported by ONR award N00014-11-1-0723 and NSF awards CCF-1139953 and CCF-11482.
The views expressed in this article are those of the authors and do not reflect the official policy or position of the United States Air Force, Department of Defense, or the U.S.\ Government.

\end{document}